\documentclass{article}

\usepackage{tikz}
\usepackage{microtype}
\usepackage{svg}
\usepackage{graphicx}
\usepackage{adjustbox}
\usepackage{subcaption}
\usepackage{booktabs} 
\usepackage{multirow}
\usepackage{stmaryrd}
\usepackage{hyperref}
\usepackage{tabularx}

\usepackage{enumitem}
\usepackage{wrapfig}

\usepackage{subcaption}

\newcommand{\method}{WeDAS}
\newcommand{\score}{QRAS}

 \usepackage[preprint]{icml2026}

\usepackage{amsmath}
\usepackage{amssymb}
\usepackage{mathtools}
\usepackage{amsthm}

\usepackage[capitalize,noabbrev]{cleveref}

\theoremstyle{plain}
\newtheorem{theorem}{Theorem}[section]
\newtheorem{proposition}[theorem]{Proposition}

\theoremstyle{definition}

\newtheorem{assumption}[theorem]{Assumption}
\theoremstyle{remark}

\usepackage[textsize=tiny]{todonotes}

\icmltitlerunning{Submission and Formatting Instructions for ICML 2026}

\begin{document}

\twocolumn[
  \icmltitle{Rethinking Deep Research from the Perspective of \\ Web Content Distribution Matching}



  \icmlsetsymbol{equal}{*}

  \begin{icmlauthorlist}
    \icmlauthor{Zixuan Yu}{equal,sysu}
    \icmlauthor{Zhenheng Tang}{equal,hkust}
    \icmlauthor{Tongliang Liu}{sydney}
    \icmlauthor{Chengqi Zhang}{polyu}
    \icmlauthor{Xiaowen Chu}{hkustgz}
    \icmlauthor{Bo Han}{hkbu}
  \end{icmlauthorlist}

  \icmlaffiliation{sysu}{School of Computer science and engineering, Sun Yat-sen University}
  \icmlaffiliation{hkust}{CSE, The Hong Kong University
of Science and Technology}
  \icmlaffiliation{sydney}{Sydney AI Centre, The University of Sydney}
  \icmlaffiliation{polyu}{Department of Data Science and Artificial Intelligence, The Hong
Kong Polytechnic University}
  \icmlaffiliation{hkustgz}{DSA Thrust, The Hong Kong University of Science and Technology (GuangZhou)}
  \icmlaffiliation{hkbu}{TMLR Group,
Department of Computer Science, Hong Kong Baptist University}

  \icmlcorrespondingauthor{Bo Han}{bhanml@comp.hkbu.edu.hk}

  \icmlkeywords{Machine Learning, ICML}

  \vskip 0.3in
]



 \printAffiliationsAndNotice{\icmlEqualContribution}

\begin{abstract}
Despite the integration of search tools, Deep Search Agents often suffer from a misalignment between reasoning-driven queries and the underlying web indexing structures. Existing frameworks treat the search engine as a static utility, leading to queries that are either too coarse or too granular to retrieve precise evidence. We propose \method, a Web Content Distribution Aware framework that incorporates search-space structural characteristics into the agent’s observation space. Central to our method is the Query-Result Alignment Score, a metric quantifying the compatibility between agent intent and retrieval outcomes. To overcome the intractability of indexing the dynamic web, we introduce a few-shot probing mechanism that iteratively estimates this score via limited query accesses, allowing the agent to dynamically recalibrate sub-goals based on the local content landscape. As a plug-and-play module, \method \ consistently improves sub-goal completion and accuracy across four benchmarks, effectively bridging the gap between high-level reasoning and low-level retrieval.
\end{abstract}

\section{Introduction}
Continuous scaling and architectural refinements in Large Language Models (LLMs) have significantly catalyzed their cognitive and reasoning capabilities \cite{achiam2023gpt,kaplan2020scaling,guo2025deepseek}. By mastering complex logic and long-horizon planning, these models have transitioned from simple text processors to autonomous agents capable of Deep Search \cite{zhu2025large} and Deep Research \cite{huang2025deep}. Such systems leverage the latent reasoning power of LLMs to navigate vast information landscapes through iterative retrieval and high-level analytical thinking, marking a paradigm shift toward multi-faceted problem-solving in open-ended domains \cite{huang2025deep}.

Despite significant strides in agentic reasoning \cite{ke2025survey,kaplan2020scaling}, a profound discrepancy persists between the internal cognitive depth of Large Language Models (LLMs) and the operational precision of their information acquisition \cite{luo2025large}. While the central intelligence of these agents has matured, their proficiency in navigating the open web remains stifled by a persistent information-to-noise bottleneck \cite{xi2025survey}. We contend that the primary impediment to effective deep research is not the agent’s core reasoning capacity, but its sensory capability to retrieve the most relevant pages from web environments \cite{qian2025scent,liu2025advances}. This bottleneck originates from a structural misalignment between an agent’s linguistic intent and the latent information distribution indexed by search engines \cite{li2025query,zhu2025large}. Such an acquisition gap is further exacerbated by the agent's inability to perceive web content distribution, leading to a pervasive failure in query granularity calibration. While coarse queries trigger a deluge of irrelevant noise, hyper-specific queries result in retrieval sparsity \cite{li2025webthinker,huang2025deep}. 

To quantify the semantic coherence between generated search queries and retrieved results, we employ TF-IDF \cite{luhn1957statistical}, Jaccard Similarity \cite{jaccard1901etude}, and Normalized
Levenshtein Similarity \cite{lcvenshtcin1966binary}, to measure their correlation. As in Figure~\ref{fig:scores}, there is a distinct distributional shift in similarity metrics between successful and failed trajectories, confirming the structural misalignment discussed above. These trajectories are collected from MiroThinker-v1.0-30B \cite{miromind2025mirothinker} when solving GAIA \cite{mialon2023gaia} tasks, where a trajectory is labeled as successful if the final answer is correct and failed otherwise. Motivated by these challenges, we investigate the following core research question:

\begin{figure}[ht]
    \centering
    \includegraphics[width=\columnwidth]{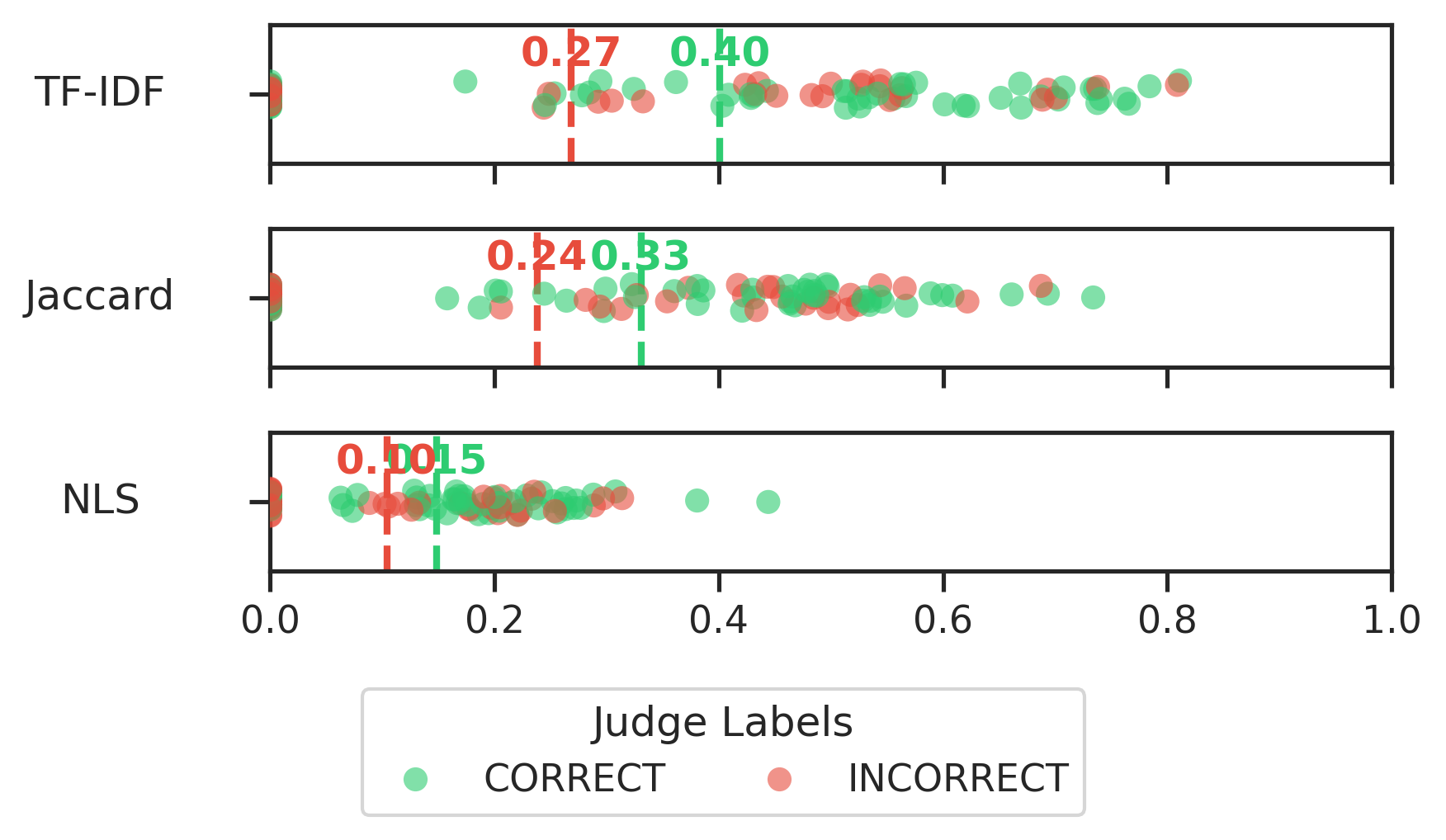} 
    \caption{Distributions of query--observation alignment metrics (TF-IDF, Jaccard, and normalized Levenshtein similarity) for successful vs. failed trajectories, highlighting the structural misalignment between agent-generated queries and retrieved web content.}
    \label{fig:scores}
\end{figure}

\textit{How can an autonomous agent perceive the distribution of web content to adaptively calibrate its search direction and query granularity?}

To address this, we introduce the Query-Result Alignment Score (\score), a metric designed to quantify the congruence between search queries and the prevailing information landscape. By decomposing the interface between queries and search outputs into three distinct dimensions, \score \ provides a granular assessment of retrieval efficacy. This metric allows agents to evaluate the searchability of their objectives and adaptively refine their trajectories based on empirical feedback.

Building on this metric, we propose Web Content Distribution Aware Search (\method), a framework that empowers agents with environmental awareness. By integrating \score, \method \ enables agents to anticipate the density and relevance of web content before finalizing their search paths. This foresight ensures that retrieved results are tightly coupled with research objectives, significantly increasing the signal-to-noise ratio in complex investigations.

The primary technical challenge to this approach is that the vast and unstructured nature of the internet precludes a priori knowledge of content distribution. We overcome this by introducing a Content Distribution Probing mechanism that samples the potential query space for a specific research sub-goal. This mechanism constructs a representative mapping of available information, providing the feedback necessary to navigate the trade-off between query breadth and specificity. By surfacing the underlying density of online data, our framework guides the agent to calibrate its search trajectories, ensuring that subsequent efforts are directed toward high-utility information regions while avoiding redundant or sparse subspaces. Our contributions are below: 
\begin{itemize}[leftmargin=*]
\item \noindent We formalize the \textbf{Query-Result Alignment Score} (\score) as a tractable metric, providing a granular metric to quantify search effectiveness (Section~\ref{sec:score}). 
\item \noindent We propose \textbf{Web Content Distribution Aware Search} (\method), a novel framework that utilizes iterative few-shot probing to map the latent information topography of the web (Section~\ref{sec:wcds}). 
\item \noindent Through experiments on open-ended research benchmarks, we demonstrate that \method \ enhances the information gain of search trajectories (Section~\ref{sec:probing}).
\end{itemize}

\begin{figure*}[t]
  \centering
  \includegraphics[width=0.8\textwidth]{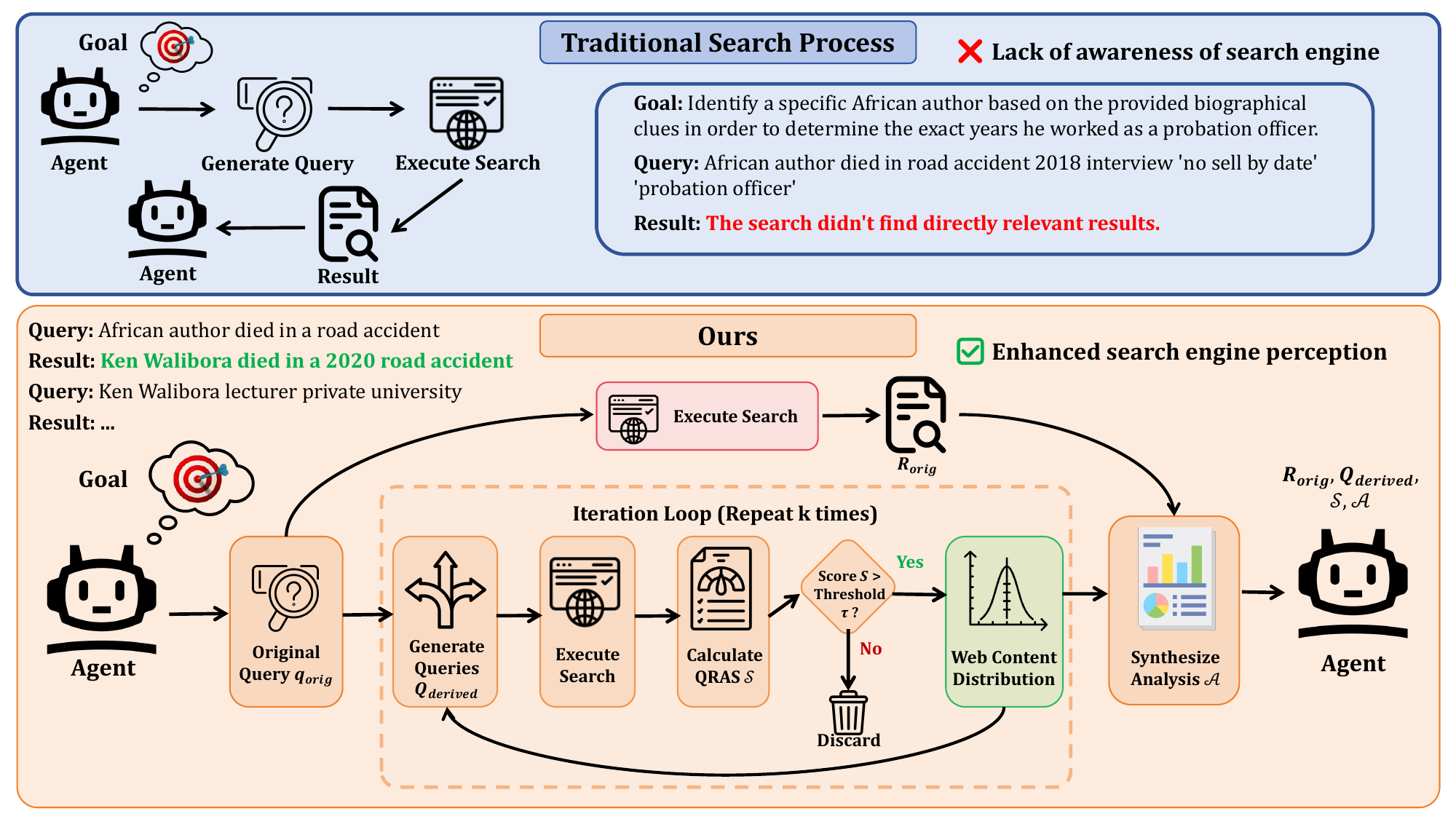}
  \vspace{-0.0cm}
  \caption{Framework of Web Content Distribution Aware Search (\method)}
  \label{fig:framework}
\end{figure*}
  \vspace{-0.0cm}

\section{Related Work}
\label{sec:related_work}
\vspace{-0.0cm}
\subsection{Search Engine}
\vspace{-0.0cm}
Traditional search engines have historically relied on lexical matching via statistical frameworks such as TF-IDF \cite{luhn1957statistical} and BM25 \cite{jones2000probabilistic}, which prioritize exact keyword overlap and term frequency \cite{zhu2025large}. These methods leverage inverted index structures to efficiently rank documents based on heuristic-based scoring \cite{li2025matching}. While computationally robust and highly scalable, these approaches are fundamentally constrained by the vocabulary mismatch problem \cite{li2025query}. Because they depend on surface-level term matching, they often fail to capture latent semantic relationships between queries and documents when they do not share identical terminology—a limitation that has prompted the transition toward neural and semantic-based retrieval paradigms \cite{thakur2021beir}.

\vspace{-0.0cm}
\subsection{Tool-Augmented Large Language Model}
\vspace{-0.0cm}

Tool-Augmented Large Language Models represent a significant shift from static generators to active problem-solvers, empowered to interact with external tools to overcome their inherent knowledge cutoff and hallucinations \cite{qu2025tool}. A core capability in this framework is the integration of web search engines, which allow LLMs to act as dynamic interfaces for real-world information acquisition \cite{zhu2025large}. Unlike traditional Retrieval-Augmented Generation (RAG) which often operates on pre-defined, local corpora \cite{huang2024survey}, search-augmented systems leverage the open web to retrieve up-to-date and long-tail knowledge.

The standard integration follows a rewrite-retrieve-read pipeline. Specifically, the model first employs query rewriting to optimize the user’s intent for search engine compatibility, executes the retrieval, and finally synthesizes the grounded response \cite{ma2023query}. However, these conventional methods are often constrained by a static, single-turn interaction schema that limits their effectiveness in handling complex, multi-step queries. Recent research is therefore pivoting toward agentic search paradigms and Deep Research architectures \cite{xi2025survey}. These systems, such as Search-R1 and WebThinker, utilize reinforcement learning and autonomous planning to perform iterative browsing and reasoning, enabling deeper information mining \cite{jin2025search,li2025webthinker}.

\vspace{-0.0cm}
\subsection{Deep Search Agent}
\vspace{-0.0cm}
Unlike traditional Tool-Augmented LLMs that typically rely on single-turn retrieval, Deep Search Agents possess the autonomy to comprehend complex user intent, plan multi-stage search strategies, and execute dynamic, multi-turn retrieval across diverse sources \cite{huang2025deep}. Current research categorizes the architectural paradigms of these agents into three distinct structures \cite{xi2025survey}:

\textbf{Parallel Structures}, which decompose intricate intents into independent sub-queries or diversify queries to cover multiple perspectives simultaneously \cite{zhao2025parallelsearch}. \textbf{Sequential Structures}, which utilize reflection-driven loops or proactive planning to adaptively refine search actions based on intermediate feedback and cumulative reasoning \cite{chen2025learning}. \textbf{Hybrid Structures}, such as tree-based or graph-based approaches, which synergize parallel exploration with sequential reasoning to optimize complex information-seeking trajectories \cite{chen2024mindsearch}.

Despite these advancements, existing methods predominantly focus on adapting models to fixed workflows while neglecting the evolution of the tool itself. Consequently, search engines often remain static components rather than becoming dynamic, co-reasoning partners in the agent’s iterative discovery process.

\vspace{-0.0cm}
\section{Preliminary}
\vspace{-0.0cm}
\subsection{Large Language Model}
\vspace{-0.0cm}
We formalize the Large Language Model (LLM) within the agentic workflow as a parameterized function $\pi_\theta$. Let $\mathcal{X}$ denote the space of natural language strings. Given a high-level research objective $x \in \mathcal{X}$ and the agent's accumulated interaction history $S_t = \{(a_i, o_i)\}_{i=1}^{t-1}$, where $a_i$ is the $i$-th tool-augmented action (e.g., issuing a query or opening a page) and $o_i$ is the corresponding observation, the LLM functions as a stochastic policy that maps the current context to a distribution over the next action $a_t$:
\begin{equation}
    a_t \sim P_{\pi_\theta}(\cdot \mid x, S_t)
\end{equation}
Specifically, the LLM processes the concatenated sequence $[x; S_t]$ as a sequence of tokens $\mathbf{t} = (t_1, t_2, \dots, t_n)$. The output generation is typically an autoregressive process, where the probability of generating an action $a_t$ is decomposed as:
\begin{equation}
    P(a_t \mid x, S_t) = \prod_{j=1}^{L} P(t_{j} \mid x, S_t, t_{<j}; \theta)
\end{equation}
where $L$ is the length of the generated action string.

In our framework, we reuse the same backbone model with different prompts: $\pi_\theta$ denotes the \emph{actor} that generates actions (when $a_t$ corresponds to a search action, it is parameterized by a query $q_t$), while a separate meta-evaluator $\mathcal{M}_\theta$ (Section~\ref{sec:wcds}) is used to compute the alignment score $\varphi(q,\tilde{o})$ and synthesize the feasibility analysis $\mathcal{A}$ based on the retrieved evidence.
\subsection{Search Engine}
\label{sec:search_engine}
The search engine acts as the interface between the agent and the open-domain environment $\mathcal{W}$. We define the environment as an unstructured, high-dimensional corpus $\mathcal{D} = \{d_1, d_2, \dots, d_N\}$, where $N$ is the total number of indexed web pages. Mathematically, the search engine is modeled as a retrieval function $f_{SE}: \mathcal{Q} \to \mathcal{O}$, where $\mathcal{Q} \subset \mathcal{X}$ is the space of possible search queries.

Given a query $q_t \in \mathcal{Q}$ generated by the agent at time $t$, the search engine returns a structured observation $o_t$, which is a top-$k$ ranked list of documents:
\begin{equation}
    o_t = f_{SE}(q_t; \mathcal{W}) = \{d^{(1)}, d^{(2)}, \dots, d^{(k)}\}.
\end{equation}
where $d^{(j)} \in \mathcal{D}$ represents the $j$-th retrieved document. Each document $d$ consists of a tuple $(\text{title, snippet, URL})$.

In our experiments, we instantiate $f_{SE}$ with the Serper Google Search API; each call returns $k=10$ results. For scoring and similarity computations, we further define a textualized observation $\tilde{o}_t$ by concatenating the returned snippets (and optionally titles) into a single string; unless otherwise specified, we use snippet-only concatenation as $\tilde{o}_t$.

In this work, we treat $f_{SE}$ as a \textbf{non-differentiable black-box function}, meaning the agent cannot access the underlying ranking gradients or the full index $\mathcal{D}$. The core challenge, therefore, lies in approximating the response distribution $P(o_t \mid q_t, \mathcal{W})$ through iterative probing to overcome the relevance bottleneck introduced by the discrete and volatile nature of the web index.
\subsection{Deep Search Agent Workflow}

We formalize a Deep Search Agent as a hierarchical decision-making system that solves an open-domain task $\mathcal{T} \in \mathcal{X}$ via iterative interaction with a dynamic web environment. The process begins with a \textit{planning phase}, where a task-decomposition operator $\Phi_{\text{plan}}$ maps the input into $n$ atomic sub-questions: $\mathcal{S} = \{s_1, s_2, \dots, s_n\} = \Phi_{\text{plan}}(\mathcal{T})$. For each sub-question $s_i$, the agent enters a \textit{reasoning--action loop}. At each step $t$, the agent maintains an interaction history $S_{t,i}$ for sub-question $s_i$, samples an intermediate reasoning trace $c_{t,i} \sim \pi_{\text{reason}}(S_{t,i}, s_i)$, and selects a tool-augmented action $a_{t,i} \sim \pi_{\text{act}}(c_{t,i})$ (e.g., issuing a query or opening a page). When $a_{t,i}$ corresponds to a search action, it is parameterized by a query $q_{t,i}$.

The environment then returns an observation $o_{t,i}$, and the history is updated as $S_{t+1,i} = \{S_{t,i}, a_{t,i}, o_{t,i}\}$. This loop terminates when a stopping criterion is met, yielding a sub-result $r_i$. Finally, in the \textit{summarization phase}, an aggregation operator $\Psi$ synthesizes all sub-results into the final answer: $A = \Psi(\mathcal{T}, \{r_1, \dots, r_n\})$. The overall procedure is summarized in Algorithm~\ref{alg:deep_search}. Throughout this section, we use $q$ to denote a search query, $o$ to denote the corresponding retrieved observation (e.g., a top-$k$ list of documents), and $S_t$ to denote the interaction history.

\begin{algorithm}[tb]
   \caption{Deep Research Agent Framework}
   \label{alg:deep_search}
\begin{algorithmic}[1]
   \STATE {\bfseries Input:} Task $\mathcal{T}$, Reasoning Policy $\pi_{\theta}$, Limit $T_{max}$
   \STATE {\bfseries Output:} Synthesized Response $A$
   
   \STATE $\{s_i\}_{i=1}^n \leftarrow \Phi_{\text{plan}}(\mathcal{T})$ \COMMENT{Task decomposition}
   \STATE $\mathcal{K} \leftarrow \emptyset$ \COMMENT{Knowledge collection}

   \FOR{$i=1$ {\bfseries to} $n$}
      \STATE $\xi \leftarrow (s_i)$ \COMMENT{Initialize trajectory}
      \WHILE{$|\xi| < T_{max}$}
         \STATE $c \sim \pi_{\theta}(\xi)$ \COMMENT{Latent reasoning state}
         \STATE $a \sim \pi_{\text{act}}(c)$
         
         \IF{$a = \text{TERMINATE}$}
            \STATE \textbf{break}
         \ENDIF
         
         \STATE $o \leftarrow \text{Env.step}(a)$ \COMMENT{Search or browse interaction}
         \STATE $\xi \leftarrow \xi \circ (a, o)$ \COMMENT{Update trajectory}
      \ENDWHILE
      \STATE $\mathcal{K} \leftarrow \mathcal{K} \cup \{\text{Summarize}(\xi)\}$
   \ENDFOR

   \STATE $A \leftarrow \Psi(\mathcal{T}, \mathcal{K})$ \COMMENT{Final information synthesis}
   \STATE \textbf{return} $A$
\end{algorithmic}
\end{algorithm}

\vspace{-0.0cm}
\section{Query-Result Alignment Score}
\label{sec:score}
\vspace{-0.0cm}

Throughout this section, we use $q$ to denote a search query, $o$ to denote the corresponding structured retrieval observation (e.g., a top-$k$ list of documents), and $\tilde{o}$ to denote its textualized form (defined in Section~\ref{sec:search_engine}) constructed from the returned snippets. We use $S_t$ to denote the interaction history.
\subsection{Definition of Information Gain}
We view a search trajectory as an evidence acquisition process that reduces uncertainty about a latent ground-truth answer random variable $A^{\star}$ (the correct solution to the task) given the agent's interaction history $S_t$. Let
\begin{equation}
H(A^{\star} \mid S_t) = - \mathbb{E}_{A^{\star} \sim P(\cdot\mid S_t)} \big[\log P(A^{\star} \mid S_t)\big]
\end{equation}
denote the conditional Shannon entropy at step $t$. After issuing a query $q_t$ and receiving an observation $o_t$ from the web environment $\mathcal{W}$, the history is updated as $S_{t+1} = \{S_t, q_t, o_t\}$. In practice, whenever we need to compute alignment-related scores, we convert the structured retrieval $o_t$ into its textualized form $\tilde{o}_t$ (Section~\ref{sec:score}) and apply $\varphi(q_t,\tilde{o}_t)$.

\textbf{Pointwise IG and expected IG.}
Given a realized observation $o_t$, we define the \textbf{pointwise information gain} as
\begin{equation}
IG(o_t; A^{\star} \mid S_t, q_t) \triangleq H(A^{\star} \mid S_t) - H(A^{\star} \mid S_t, q_t, o_t).
\end{equation}
Since the agent chooses $q_t$ before seeing $o_t$, the decision objective is the \textbf{expected information gain}. Here, $q_t$ is a decision variable used to index the observation distribution $P(o_t \mid q_t, \mathcal{W})$, rather than an additional source of information about $A^{\star}$.
\begin{equation}
\mathrm{EIG}(q_t \mid S_t) \triangleq \mathbb{E}_{o_t \sim P(o_t \mid q_t, \mathcal{W})}\big[IG(o_t; A^{\star} \mid S_t, q_t)\big].
\end{equation}
An optimal policy $\pi_\theta$ therefore aims to select queries that maximize $\mathrm{EIG}(q_t \mid S_t)$. In open-domain settings, however, the posterior $P(A^{\star}\mid S_t, q_t, o_t)$ is non-analytic and high-dimensional, making direct estimation of $\mathrm{EIG}$ intractable; moreover, noisy retrieval induces a persistent \textit{relevance bottleneck}.

\textbf{KL form.}
Pointwise information gain measures entropy reduction and can be negative for a particular realization. However, it is well known that the \emph{expected} information gain is equivalent to the expected KL divergence (Bayesian surprise) between the posterior and the (query-conditioned) prior:
\begin{equation}
\resizebox{\hsize}{!}{$
\begin{split}
\mathrm{EIG}(q_t \mid S_t) &= \mathbb{E}_{o_t \sim P(\cdot \mid q_t, \mathcal{W})}\Big[D_{\mathrm{KL}}\big(P(A^{\star} \mid S_t, q_t, o_t) \,\|\, P(A^{\star} \mid S_t)\big)\Big] \\
&= I(A^{\star}; o_t \mid S_t, q_t).
\end{split}
$}
\end{equation}
Since $q_t$ is selected before observing $o_t$ and does not itself reveal new evidence about $A^{\star}$, we have $P(A^{\star}\mid S_t, q_t)=P(A^{\star}\mid S_t)$.
Consequently, maximizing $\mathrm{EIG}$ amounts to finding queries that are expected to induce large posterior shifts.

\textbf{Latent relevance and proxy objective.}
We introduce a binary latent variable $z \in \{0, 1\}$ indicating whether the retrieved observation contains \emph{novel task-relevant evidence} ($z=1$) or mostly noise / redundancy ($z=0$).

\begin{assumption}[Query chosen before observation]
\label{ass:query_no_info}
The query $q_t$ is selected before observing $o_t$, and does not itself reveal new evidence about $A^{\star}$. Hence $P(A^{\star}\mid S_t, q_t)=P(A^{\star}\mid S_t)$.
\end{assumption}

\begin{assumption}[Relevance-conditioned posteriors]
\label{ass:z_posterior}
There exists a random variable $z$ such that
\begin{equation}
\small
P(A^{\star} \mid S_t, q_t, o_t) = \sum_{z \in \{0,1\}} P(A^{\star} \mid S_t, q_t, o_t, z)\, P(z \mid S_t, q_t, o_t).
\end{equation}
Moreover, when $z=0$ (no novel evidence), the observation produces a negligible posterior shift:
$P(A^{\star} \mid S_t, q_t, o_t, z=0) \approx P(A^{\star} \mid S_t, q_t)$.
\end{assumption}

\begin{assumption}[Bounded conditional posterior shift]
\label{ass:delta_bounds}
Let $\Delta(o_t, z) \triangleq D_{\mathrm{KL}}\big(P(A^{\star}\mid S_t,q_t,o_t, z) \,\|\, P(A^{\star}\mid S_t)\big)$ denote the posterior shift conditioned on the latent relevance $z$.
We assume that:
(1) If $z=0$, $\Delta(o_t, z=0) \approx 0$.
(2) If $z=1$, there exist constants $\Delta_{\min}, \Delta_{\max} > 0$ such that $\Delta(o_t, z=1) \in [\Delta_{\min},\Delta_{\max}]$.
\end{assumption}

\begin{proposition}[EIG Upper Bound via Relevance]
\label{prop:eig_bounds}
Under Assumptions~\ref{ass:query_no_info}--\ref{ass:delta_bounds}, the Expected Information Gain is bounded above by the expected relevance:
\begin{equation}
\small
\mathrm{EIG}(q_t\mid S_t) \;\le\; \Delta_{\max}\, \mathbb{E}_{o_t \sim P(o_t\mid q_t,\mathcal{W})}\big[P(z=1\mid S_t,q_t,o_t)\big].
\end{equation}
\end{proposition}

\begin{proof}[Proof sketch]
Recall that the Expected Information Gain is defined as $\mathrm{EIG}(q_t \mid S_t) = \mathbb{E}_{o_t}[D_{\mathrm{KL}}(P(A^{\star} \mid o_t) \parallel P(A^{\star}))]$.

By Assumption~\ref{ass:z_posterior}, the posterior $P(A^{\star} \mid o_t)$ is a mixture distribution $\sum_z P(z \mid o_t)P(A^{\star} \mid o_t,z)$. Since the KL divergence $D_{\mathrm{KL}}(\cdot \parallel Q)$ is a convex function, we apply Jensen's inequality to obtain:
\begin{align*}
    D_{\mathrm{KL}}(P(A^{\star} \mid o_t) \parallel P(A^{\star})) 
    &\le \sum_{z \in \{0,1\}} P(z \mid o_t) \, \Delta(o_t, z).
\end{align*}
Applying Assumption~\ref{ass:delta_bounds}, the term corresponding to $z=0$ vanishes (as $\Delta \approx 0$), while the $z=1$ term is bounded by $\Delta_{\max}$. Thus:
\[
    D_{\mathrm{KL}}(\dots) \le P(z=1 \mid o_t) \Delta_{\max}.
\]
Finally, taking the expectation over $o_t$ on both sides completes the proof.
\end{proof}

\textbf{From $o_t$ to $\tilde{o}_t$ and to $\varphi$.}
In practice we only observe a textualized summary $\tilde{o}_t$ (Section~\ref{sec:search_engine}) rather than the full structured $o_t$.
We therefore approximate $P(z=1\mid S_t,q_t,o_t)$ by $P(z=1\mid S_t,q_t,\tilde{o}_t)$.
Finally, we use a bounded alignment score $\varphi(q_t,\tilde{o}_t)$ as a \emph{rank-consistent} proxy for this relevance probability (i.e., larger $\varphi$ should indicate larger $P(z=1\mid S_t,q_t,\tilde{o}_t)$), and use it for filtering and trajectory control.

\subsection{Measuring the information gain of the search process}

We employ three complementary metrics to measure the alignment between an agent-generated query $q$ and the retrieved observation. Concretely, we compute similarities between $q$ and the textualized observation $\tilde{o}$ (the concatenation of the top-$k$ returned snippets): (i) \textbf{TF-IDF similarity}, defined as the cosine similarity between weighted term vectors, $\text{Sim}_{\text{tf-idf}}(q, \tilde{o}) = \frac{\mathbf{v}_q \cdot \mathbf{v}_{\tilde{o}}}{\|\mathbf{v}_q\| \|\mathbf{v}_{\tilde{o}}\|}$; (ii) \textbf{Jaccard similarity}, which quantifies token overlap, $J(q, \tilde{o}) = \frac{|T_q \cap T_{\tilde{o}}|}{|T_q \cup T_{\tilde{o}}|}$; and (iii) \textbf{Normalized Levenshtein similarity (NLS)}, which captures morphological proximity via normalized edit distance, $\text{Sim}_{\text{nls}}(q, \tilde{o}) = 1 - \frac{\text{Lev}(q, \tilde{o})}{\max(|q|, |\tilde{o}|)}$. Empirically, the utility of an individual search step is often positively correlated with these alignment scores.

Although the metrics detailed above provide a robust statistical framework for evaluating query-content correspondence, they possess inherent limitations when integrated into a deep search agent framework. Specifically, these similarities function as \textit{ex-post} descriptive statistics rather than \textit{ex-ante} predictive heuristics. While they can quantify the degree of information sparsity or noise, they lack the intrinsic mechanism to provide the agent with directional information for query optimization. Consequently, relying solely on these conventional metrics prevents the agent from autonomously calibrating its granularity in response to the dynamic content density of the web.

To address this limitation and formalize a feedback signal, we introduce the \textbf{Query-Result Alignment Score (\score)} as a bounded utility metric $\varphi(q,\tilde{o}) \in \llbracket 0, 10 \rrbracket$, where $\tilde{o}$ denotes the textualized observation constructed from the top-$k$ retrieved snippets. In our implementation, the meta-evaluator takes as input the query together with the top-$k$ retrieved snippets, and outputs an overall score along with a short analysis. Rather than approximating global entropy reduction directly, $\varphi$ provides a local, multi-faceted assessment of search quality via three latent dimensions:
\begin{itemize}
    \item \textbf{Topical relevance ($s_{\mathrm{rel}}$):} semantic congruence between the query $q$ and the textualized observation $\tilde{o}$.
    \item \textbf{Information density ($s_{\mathrm{den}}$):} concentration of non-redundant, task-salient propositions in $\tilde{o}$.
    \item \textbf{Noise robustness ($s_{\mathrm{noi}}$):} inverse prevalence of irrelevant or distractive content in $\tilde{o}$.
\end{itemize}
We define the overall score as the unweighted mean of the three sub-scores:
\begin{equation}
\varphi(q, \tilde{o}) = \frac{1}{3}\big(s_{\mathrm{rel}}(q,\tilde{o}) + s_{\mathrm{den}}(q,\tilde{o}) + s_{\mathrm{noi}}(q,\tilde{o})\big).
\end{equation}
This decomposition yields an interpretable proxy for per-step search utility and enables systematic evaluation of exploration efficiency without requiring full rollouts.

\section{Methodology}
\subsection{Web Content Distribution Aware Search}
\label{sec:wcds}

Building upon the proposed alignment score, we introduce \textbf{Web Content Distribution-Aware Search (\method)} to enable agents to dynamically calibrate their exploration strategy. \method \ integrates a feedback loop where a Large Language Model (LLM) functions as a meta-evaluator $\mathcal{M}_\theta$. 

As outlined in Algorithm~\ref{alg:wcda_search}, for each query--observation pair $(q, \tilde{o})$, $\mathcal{M}_\theta$ computes the alignment score $\varphi(q,\tilde{o})$ and produces a qualitative feasibility analysis $\mathcal{A}$ for the current trajectory. By augmenting raw observations with these meta-evaluations, the agent obtains an empirical estimate of the \textit{local web content distribution}. This awareness enables adaptive control of query granularity, dynamically switching between broad exploratory probes and targeted queries, thereby improving search efficiency under a fixed interaction budget.

\textbf{Meta-evaluator interface.}
In our implementation, the meta-evaluator $\mathcal{M}_\theta$ performs \textbf{holistic scoring}: it takes as input a query $q$ together with the corresponding textualized observation $\tilde{o}$ (a concatenation of the top-$k$ search snippets), and outputs three sub-scores $\big(s_{\mathrm{rel}}(q,\tilde{o}), s_{\mathrm{den}}(q,\tilde{o}), s_{\mathrm{noi}}(q,\tilde{o})\big) \in [0,10]^3$ and their mean $\varphi(q,\tilde{o}) \in [0,10]$.

\textbf{Dynamic thresholding.}
Instead of using a fixed score threshold, we maintain a small set of high-utility probe queries under a fixed probing budget, while always keeping the initial query $q_0$ as the main execution query.

Let $\mathcal{P}_t$ denote the set of derived probe tuples collected up to iteration $t$, where each tuple is $(q, s, \mathcal{A})$ with $s=\varphi(q,\tilde{o})$ and $\tilde{o}$ is the textualized observation returned by searching query $q$. We set
\begin{equation}
\tau_t = \min_{(q,s,\mathcal{A}) \in \mathcal{P}_t} s,
\end{equation}
and remove the lowest-scored probe query. This procedure can be viewed as an online top-set maintenance rule: at each iteration we discard the current worst probe and keep the remaining higher-scored probes as guidance candidates.

\textbf{Outputs to the agent.}
During probing, each derived query returns a tuple $(q, s, \mathcal{A})$ containing the query, its alignment score, and the meta-evaluator analysis, which is used as a guidance signal for subsequent query generation. In contrast, the agent only consumes the retrieved search results from the initial query $q_0$ for grounding and evidence collection.

\textbf{Candidate generation operator.}
In Algorithm~\ref{alg:wcda_search}, $\Gamma_{\text{gen}}(\cdot)$ denotes a candidate-generation operator (implemented by the LLM) that expands the current query set into semantically diverse variants for probing.

\begin{algorithm}[tb]
   \caption{\method}
   \label{alg:wcda_search}
\begin{algorithmic}[1]
   \STATE {\bfseries Input:} Initial query $q_0$, Max probe iterations $T$
   \STATE $\mathcal{P} \leftarrow \emptyset$ \COMMENT{Collected probe tuples}

   \FOR{$t = 1$ {\bfseries to} $T$}
      \STATE $q_t \leftarrow \Gamma_{\text{gen}}(q_0, \mathcal{P})$ \COMMENT{Generate one derived probe query}
      \STATE $\tilde{o}_t \leftarrow \text{Search}(q_t)$ \COMMENT{Returns textualized snippets}
      \STATE $s_t \leftarrow \varphi(q_t, \tilde{o}_t)$
      \STATE $\mathcal{A}_t \leftarrow \mathcal{M}_\theta(q_t, \tilde{o}_t)$ \COMMENT{Short analysis from meta-evaluator}
      \STATE $\mathcal{P} \leftarrow \mathcal{P} \cup \{(q_t, s_t, \mathcal{A}_t)\}$

      \STATE $(q^{-}, s^{-}, \mathcal{A}^{-}) \leftarrow \arg\min_{(q,s,\mathcal{A}) \in \mathcal{P}} s$
      \STATE $\mathcal{P} \leftarrow \mathcal{P} \setminus \{(q^{-}, s^{-}, \mathcal{A}^{-})\}$ \COMMENT{Drop one lowest-scored probe}
   \ENDFOR

   \STATE \textbf{return} $\text{Synthesize}(\mathcal{P})$ \COMMENT{Return guidance tuples only (not evidence)}
\end{algorithmic}
\end{algorithm}
\subsection{Distribution Estimation via Few-Shot Query Probing}
\label{sec:probing}
To address uncertainty in query optimization, we propose an \textbf{iterative few-shot probing} strategy to approximate the local distribution of high-utility information. Since the scale and heterogeneity of the web render global modeling of $P(o \mid q, \mathcal{W})$ intractable, we instead estimate a \textit{local} information topography via the following steps.

\textbf{Candidate generation.} Given the current query set $\mathcal{Q}_{\text{cur}}$, the agent produces semantically diverse variants $\mathcal{Q}_{\text{cand}}$ to probe different regions of the query space.

\textbf{Alignment evaluation and pruning.} For each candidate $q \in \mathcal{Q}_{\text{cand}}$, the agent executes a search to obtain $o$ and computes $s = \varphi(q, o)$. We then apply threshold-based pruning:
\begin{equation}
\mathcal{Q}_{\text{next}} = \{ q \in \mathcal{Q}_{\text{cand}} \mid \varphi(q, o) > \tau \}.
\end{equation}
This filtering ensures that only candidates with high estimated utility are expanded at iteration $t+1$.

\textbf{Synthesis and meta-analysis.} Upon reaching the iteration limit $T$ or convergence, the agent invokes $\text{Synthesize}(\cdot)$ to aggregate probed tuples $(q, o, \varphi(q,o))$ and produce a qualitative analysis $\mathcal{A}$. The resulting distribution-aware context helps reduce retrieval noise and improves subsequent decision making.
\section{Experiments}

\subsection{Setup}
\textbf{Benchmarks.} To empirically validate the effectiveness of the proposed framework, we evaluate our method across four representative benchmarks: (1) \textbf{Browsecomp} \cite{wei2025browsecomp} is designed to evaluate an agent's proficiency in retrieving elusive, multi-faceted information through sophisticated web browsing and reasoning strategies. (2) \textbf{Browsecomp-zh} \cite{zhou2025browsecomp} is an extension of Browsecomp focused on Chinese web content. (3) \textbf{GAIA} \cite{mialon2023gaia} centers on multi-modal reasoning and tool-use proficiency. Following the previous work \cite{li2025webthinker,wu2025webdancer,li2025websailor}, we only use a subset of 103 cases from the text-only validation split. (4) \textbf{xbench-ds} \cite{chen2025xbench} is a sophisticated, task-oriented benchmark designed to evaluate the proficiency of AI agents in autonomous tool invocation, with a specific focus on exhaustive information retrieval and multi-step search reasoning.

\textbf{Baseline Agentic Framework.} We adopt Miroflow \cite{2025mirothinker}, a sophisticated agentic framework engineered for structured task decomposition and multi-tool orchestration. Miroflow provides a modular environment that allows the agent to maintain high logical consistency across long-horizon tasks, making it particularly effective for the complex retrieval and reasoning requirements of our target benchmarks. In our experimental tables, we compare against both \emph{direct inference} baselines (e.g., Qwen-2.5-32B/72B \cite{qwen3technicalreport}, GPT-4o/GPT-4.1 \cite{achiam2023gpt}, QwQ-32B, o4-mini) and representative \emph{agent systems} / open-source search agents (e.g., WebSailor \cite{li2025websailor}, Search-o1 \cite{li2025search}, WebThinker \cite{li2025webthinker}, ASearcher-Web \cite{gao2025beyond}, WebDancer \cite{wu2025webdancer}), as well as closed-source agents when available.

\textbf{Models.} To evaluate the generalizability and robustness of our framework across different model architectures and alignment paradigms, we select two state-of-the-art LLMs as the main agent: (1) \textbf{MiroThinker-v1.0-30B} \cite{miromind2025mirothinker} is an open-source initialized from the Qwen3-30B-A3B \cite{qwen3technicalreport} base architecture. Through targeted fine-tuning and specialized post-training, the model is engineered to excel in agentic reasoning and iterative verification. MiroThinker-v1.0-30B is strategically optimized for seamless integration and architectural alignment with the Miroflow agentic framework. (2) \textbf{GPT-5-mini} \cite{singh2025openai} is a  state-of-the-art closed-source model from OpenAI, functioning as a general-purpose reasoning engine. In our experiments, GPT-5-mini is integrated directly into the framework without domain-specific adaptation. Detailed experimental settings and the prompts used in our pipeline are provided in Appendix~\ref{sec:appendix-exp-setup} and Appendix~\ref{sec:appendix-prompts}.

\subsection{Main Result}

Table~\ref{tab:framework-comparison} summarizes the main results under a unified agentic framework, where we swap the probing module while keeping the rest of the workflow fixed. Across both backbones, integrating \method \ consistently improves pass@3 on search-intensive benchmarks, demonstrating that the proposed distribution-aware probing is largely model-agnostic.
\begin{table}[t]
\setlength{\abovedisplayskip}{-1pt}
\setlength{\abovecaptionskip}{-1pt}
\caption{Performance comparison across multiple frameworks and base models. Results are evaluated using the \textbf{pass@3} metric. Our method (\method) is integrated across different architectures to demonstrate its plug-and-play effectiveness. Bold values indicate the best performance within each model category. \textit{Abbreviations:} BC represents BrowseComp, BC-zh represents BrowseComp-zh, XBD represents XBench-DS, and Miro-30B represents MiroThinker-v1.0-30B.}
\label{tab:framework-comparison}
\vskip 0.15in
\begin{center}
\setlength{\tabcolsep}{2pt}
\begin{small}
\begin{sc}
\begin{adjustbox}{max width=\columnwidth}
\begin{tabular}{llcccc}
\toprule
\textbf{Framework} & \textbf{Model} &  \textbf{BC} & \textbf{BC-zh} & \textbf{GAIA} & \textbf{XBD} \\
\midrule
\multirow{2}{*}{WebSailor} & \multirow{1}{*}{WebSailor-32B}  & 16.05 & 39.28&  64.72 & 68.00 \\

 & \multirow{1}{*}{WebSailor-72B}  & 18.96 & 42.14& 66.85& 69.00 \\
\midrule
\multirow{2}{*}{\shortstack{MiroFlow}} & \multirow{1}{*}{Miro-30B} &  33.00 & 53.00 & 74.76 & \textbf{89.00}\\
 & \multirow{1}{*}{GPT-5-mini}  & 29.00 & 44.00 & 73.79 & 61.00 \\
\midrule
\multirow{2}{*}{\textbf{\method}} & Miro-30B  & \textbf{35.00} & \textbf{58.00} & \textbf{75.73} & 86.00 \\
 & GPT-5-mini  & 30.00 & 43.00 & 74.76 & 70.00 \\
\bottomrule
\end{tabular}
\end{adjustbox}
\end{sc}
\end{small}
\end{center}
\vskip -0.1in
\end{table}

\begin{table}[t]
\setlength{\abovedisplayskip}{-1pt}
\setlength{\abovecaptionskip}{-1pt}
\setlength{\tabcolsep}{2pt}
\caption{Pass@1 accuracy (\%) performance across four benchmarks. We compare our proposed framework against direct inference baselines and representative open/closed-source agent systems. Bold indicates the best performance in each column. \textit{Abbreviations:} BC represents BrowseComp, BC-zh represents BrowseComp-zh, XBD represents XBench-DS, and Miro-30B represents MiroThinker-v1.0-30B.}
\label{tab:full-benchmark-results}
\vskip 0.15in
\begin{center}
\begin{small}
\begin{sc}
\begin{adjustbox}{max width=\columnwidth}
\begin{tabular}{llcccc}
\toprule
\textbf{Category} &\textbf{ Model / Framework} & \textbf{BC} & \textbf{BC-zh} & \textbf{GAIA} & \textbf{XBD}\\
\midrule
\multirow{6}{*}{\shortstack{\textbf{Direct}\\\textbf{Inference}}} 
    & Qwen-2.5-32B & 0.6& 3.9 &13.6 &8.7\\
    & Qwen-2.5-72B & 0.6 &7.0 & 14.6 &12.7 \\
    & GPT-4o  & 0.6 &6.2 &17.5&18.0 \\
    & GPT-4.1 & 1.5 &14.4 & 22.3 &17.0  \\
    & QwQ-32B & 0.5 &10.0& 22.3 &10.7  \\
    & o4-mini  & \textbf{6.1} &\textbf{15.2} &\textbf{33.3} &\textbf{22.3}  \\
\midrule
\multirow{3}{*}{\shortstack{\textbf{Closed-}\\\textbf{Source}}} 
    &  Grok-DeepResearch  & - & 12.9 & - & 52.0 \\
    & Doubao-DeepThink & - & 26.0 & - & 50.0 \\
    & OpenAI DeepResearch  & \textbf{51.5} & \textbf{42.9} & \textbf{67.4} & - \\
\midrule
\multirow{10}{*}{\shortstack{\textbf{Open-}\\\textbf{Source}}} 
    & Search-o1-32B & 2.8 & 17.9& 39.8 &25.0\\
    & WebThinker-32B-RL & 2.8 & 7.3&48.5 &24.0\\
    & ASearcher-Web-32B & 5.2 &15.6 &52.8 &42.1\\
    & WebDancer-QwQ-32B & 3.8& 18.0 & 51.5 &38.3\\

    & WebSailor-32B            & 10.50 &25.50 &53.20& 53.30 \\
    & WebSailor-72B          & 12.00 & 30.10 & 55.40 & 55.00 \\
    
    & GPT-5-mini + MiroFlow & 15.00 & 28.00 & 51.46 & 38.00 \\
    & GPT-5-mini + \method & 17.00 & 25.00 & 57.28 & 44.00 \\
    & Miro-30B + MiroFlow  & 17.00 & 34.00 & 63.11 & \textbf{73.00} \\
    & Miro-30B + \method & \textbf{26.00} & \textbf{41.00} & \textbf{66.99} & 72.00 \\
\bottomrule
\end{tabular}
\end{adjustbox}
\end{sc}
\end{small}
\end{center}
\vspace{-0.0cm}
\end{table}
\vspace{-0.0cm}

Table~\ref{tab:framework-comparison} indicates that \method \ generally boosts performance, though gains vary across configurations. MiroThinker-v1.0-30B thrives in noisy, multilingual settings (e.g., BrowseComp-zh), despite a regression on XBench-DS. Overall, Table~\ref{tab:full-benchmark-results} demonstrates that \method \ significantly empowers open-source agents, offering a competitive alternative to closed-source systems.
\subsection{Analysis}

To further understand \emph{why} \method \ alleviates the information-to-noise bottleneck, we conduct a fine-grained analysis on GAIA by measuring query--observation alignment between generated queries and the retrieved results.

As reported in Table~\ref{tab:alignment-results}, \method \ consistently outperforms the baseline across all similarity dimensions (TF-IDF, Jaccard, and Normalized Levenshtein Similarity). Notably, this improvement is pervasive, manifesting in both successful and failed trials. In successful cases, our method achieves a significant boost in all metrics, suggesting that the iterative probing mechanism effectively steers the agent toward generating queries that are more semantically congruent with the actual web content distribution. 
Crucially, the performance in failure cases provides even more compelling evidence of our method's robustness. While the baseline alignment scores drop in failed trials, \method \ maintains relatively high scores on all three metrics. This disparity reveals that even when the final answer is not retrieved, \method \ still enables the agent to maintain a coherent sense of the environment, preventing it from drifting into irrelevant noise. These results empirically support our hypothesis that \method \ effectively bridges the gap by dynamically calibrating the granularity of the query based on real-time feedback from the web response.
\begin{table}[!ht]
\setlength{\abovedisplayskip}{-1pt}
\setlength{\abovecaptionskip}{-1pt}
\caption{Comparison of Different Metrics of Query-Result Alignment between Baseline and \method. }
\label{tab:alignment-results}
\vskip 0.15in
\begin{center}
\resizebox{0.9\columnwidth}{!}{ 
\begin{small}
\begin{tabular}{llccc}
\toprule
\textbf{Method} & \textbf{Outcome} & \textbf{TF-IDF} $\uparrow$ & \textbf{Jaccard} $\uparrow$ & \textbf{NLS} $\uparrow$ \\
\midrule
Baseline & Success & 0.401 & 0.331 & 0.148 \\
         & Failure & 0.269 & 0.238 & 0.104 \\
\addlinespace[0.5em]
\textbf{\method} & \textbf{Success} & \textbf{0.431} & \textbf{0.364} & \textbf{0.165} \\
                     & \textbf{Failure} & \textbf{0.292} & \textbf{0.241} & \textbf{0.108} \\
\bottomrule
\end{tabular}
\end{small}
} 
\end{center}
\end{table}
\vspace{-0.0cm}

\begin{table}[ht]
\setlength{\abovedisplayskip}{-1pt}
\setlength{\abovecaptionskip}{-1pt}
\caption{Ablation on GAIA with different max probe iterations $T$ (GPT-5-mini as agent model).}
\label{tab:ablation-probe-iter}
\vskip 0.15in
\begin{center}
\begin{small}
\begin{tabular}{ccc}
\toprule
$T$ & pass@1 & pass@3  \\
\midrule
0 & 49.51 & 72.82 \\
1 & \textbf{57.28} & \textbf{74.76} \\
2 & 54.37 & \textbf{74.76} \\
\bottomrule
\end{tabular}
\end{small}
\end{center}
\vspace{-0.1cm}
\end{table}
\vspace{-0.1cm}

\subsection{Ablation Study}
\vspace{-0.0cm}

In this section, we investigate the sensitivity of \method \ to the max probe iterations $T$, which governs the depth of the distribution-aware probing process. Unless otherwise specified, we use GPT-5-mini as the agent backbone and conduct the ablation on the GAIA benchmark. We vary $T \in \{0, 1, 2\}$ while keeping the score threshold $\tau$ constant.

Overall, enabling probing ($T\ge 1$) improves performance compared to no probing ($T=0$), which supports our claim that distribution-aware probing helps the agent calibrate query granularity and reduce retrieval noise. Increasing $T$ from $1$ to $2$ yields similar pass@3 but a slightly lower pass@1, suggesting diminishing returns under a fixed interaction budget.

\vspace{-0.1cm}
\section{Conclusion}
\vspace{-0.1cm}
We revisit deep research from the perspective of the structural mismatch between reasoning-driven query generation and the web content distribution implicitly shaped by search engine indexing, and propose \method  \ as a plug-and-play framework that endows agents with distribution awareness via iterative few-shot probing and a meta-evaluation signal. Central to the framework is the Query-Result Alignment Score (\score), which provides an actionable, multi-dimensional assessment of how well retrieved results align with an agent’s intent; experiments suggest that this distribution-aware design can improve performance in many settings by reducing the information-to-noise bottleneck, though gains may vary across models and tasks.




\newpage

\section*{Impact Statement}

This paper presents work whose goal is to advance the field of Machine Learning by improving the efficiency and accuracy of autonomous deep research agents. By aligning agentic reasoning with the structural distribution of web content, our framework reduces the computational overhead and noise associated with large-scale information retrieval. While these advancements enhance the utility of AI in academic and professional research, they also necessitate continued vigilance regarding the potential for automated systems to propagate or aggregate biased web content. Beyond these considerations, there are many potential societal consequences of our work, none of which we feel must be specifically highlighted here.




\nocite{langley00}

\bibliography{example_paper}
\bibliographystyle{icml2026}

\newpage
\appendix
\onecolumn
\section*{Appendix}

\section{Experimental Setup}
\label{sec:appendix-exp-setup}

Due to the computational cost, we do not evaluate on the full BrowseComp and BrowseComp-zh test sets; instead, we uniformly sample $100$ instances from each benchmark and report results on these subsets. For all experiments, we adopt a unified decoding configuration with temperature $=1.0$, top-$p=0.95$, min-$p=0.0$, top-$k=-1$, and max\_tokens $=16384$. The maximum number of interaction turns for main agent execution is set to $600$, with at most $10$ tool calls per turn. For our method, both the meta-evaluator (\score) and the probe query generation module are instantiated with \texttt{GPT-5-mini}. All experiments are run on a single NVIDIA RTX PRO 6000 GPU.

\section{Prompts Used in \method}
\label{sec:appendix-prompts}

To facilitate reproducibility, we provide the exact system instructions (prompts) used in our framework.

\subsection{Probe Query Generation}
\label{subsec:prompt-generation}

The following prompt is employed by the candidate generation operator $\Gamma_{\text{gen}}$. It instructs the LLM to act as a "Search Keyword Specialist," transforming the analysis of current search results into a set of tactical, keyword-based follow-up queries.

\begin{figure}[h]
\centering
\fbox{
\begin{minipage}{0.95\textwidth}
\small\ttfamily 
You are a Search Keyword Specialist. Your task is to analyze a "Query-Result" pair and determine the next tactical search keywords.

\vspace{0.8em}
\textbf{\#\#\# TASK} \\
Evaluate the result. If it reveals new information, generate 2-5 follow-up queries. If it is a dead end or fully answered, return an empty list.

\vspace{0.8em}
\textbf{\#\#\# STRICT QUERY CONSTRAINTS}
\begin{itemize}
    \setlength\itemsep{0em}
    \setlength\leftskip{1em} 
    \item \textbf{**Brevity is Mandatory**}: Each query must be a short phrase or a list of 1-3 keywords.
    \item \textbf{**No Sentences**}: Do NOT generate natural language questions (e.g., "How do I...").
    \item \textbf{**Keyword Focus**}: Use specific entities, technical terms, error codes, or unique identifiers found in the result.
    \item \textbf{**Word Limit**}: Maximum 6 words per query.
\end{itemize}

\textbf{\#\#\# DECISION LOGIC}
\begin{enumerate}
    \setlength\itemsep{0em}
    \setlength\leftskip{1em}
    \item \textbf{**Prune (Return [])**}: 
    \begin{itemize}
       \item The result is irrelevant, low quality, or a "404/Not Found" page.
       \item The result already fully satisfies the information need.
    \end{itemize}
    \item \textbf{**Expand (Return [Short Keywords])**}: 
    \begin{itemize}
       \item The result mentions a specific library, API, or concept that needs investigation.
       \item The result points to a more promising direction using professional terminology.
    \end{itemize}
\end{enumerate}

\textbf{\#\#\# OUTPUT FORMAT} \\
Return a JSON object: \\
\{ \\
\hspace*{1.5em} "analysis": "Brief note on why these keywords were chosen.", \\
\hspace*{1.5em} "derived\_queries": [ \\
\hspace*{3em} "keyword1 keyword2", \\
\hspace*{3em} "specific\_term site:github.com" \\
\hspace*{1.5em} ] \\
\}
\end{minipage}
}
\caption{System instruction for the Candidate Generation Operator ($\Gamma_{\text{gen}}$).}
\label{fig:prompt-update}
\end{figure}

\newpage

\subsection{Alignment Scoring and Meta-Evaluation}
\label{subsec:prompt-scoring}

The prompt below is used by the meta-evaluator $\mathcal{M}_\theta$ to compute the Query-Result Alignment Score (\score).

\begin{figure}[h]
\centering
\fbox{
\begin{minipage}{0.95\textwidth}
\small\ttfamily 
\textbf{\#\#\# ROLE} \\
You are a Content Relevance Evaluator. Score how well multiple search result sets align with their respective search queries.

\vspace{0.8em}
\textbf{\#\#\# INPUT} \\
You will receive a list of items. Each item contains a "query" and "results".

\vspace{0.8em}
\textbf{\#\#\# SCORING DIMENSIONS (0-10 each)}
\begin{enumerate}
    \setlength\itemsep{0em}
    \setlength\leftskip{1em}
    \item \textbf{**Topical Relevance**}: Do the results actually discuss the subject matter?
    \item \textbf{**Information Density**}: Do the snippets contain potential answers (numbers, dates, scores) or just generic info?
    \item \textbf{**Noise Level**}: How much of the result set is irrelevant (ads, SEO spam)? (0 is pure noise, 10 is zero noise)
\end{enumerate}

\textbf{\#\#\# OUTPUT FORMAT (JSON ONLY)} \\
Return a JSON object with a key "evaluations" containing a list of objects: \\
\{ \\
\hspace*{1.5em} "evaluations": [ \\
\hspace*{3em} \{ \\
\hspace*{4.5em} "query": <query>, \\
\hspace*{4.5em} "overall\_relevance\_score": <0-10>, \\
\hspace*{4.5em} "dimension\_scores": \{ \\
\hspace*{6em} "topical\_relevance": <0-10>, \\
\hspace*{6em} "info\_density": <0-10>, \\
\hspace*{6em} "noise\_level": <0-10> \\
\hspace*{4.5em} \}, \\
\hspace*{4.5em} "analysis": "<short summary>" \\
\hspace*{3em} \}, \\
\hspace*{3em} ... \\
\hspace*{1.5em} ] \\
\}
\end{minipage}
}
\caption{System instruction for the Meta-Evaluator ($\mathcal{M}_\theta$).}
\label{fig:prompt-sim}
\end{figure}


\end{document}